\def\year{2022}\relax
\documentclass[letterpaper]{article} % DO NOT CHANGE THIS
\usepackage{aaai22}  % DO NOT CHANGE THIS
\usepackage{times}  % DO NOT CHANGE THIS
\usepackage{helvet}  % DO NOT CHANGE THIS
\usepackage{courier}  % DO NOT CHANGE THIS
\usepackage[hyphens]{url}  % DO NOT CHANGE THIS
\usepackage{graphicx} % DO NOT CHANGE THIS
\usepackage{natbib}  % DO NOT CHANGE THIS AND DO NOT ADD ANY OPTIONS TO IT
\usepackage{caption} % DO NOT CHANGE THIS AND DO NOT ADD ANY OPTIONS TO IT
\DeclareCaptionStyle{ruled}{labelfont=normalfont,labelsep=colon,strut=off} % DO NOT CHANGE THIS
\usepackage[linesnumbered,algoruled,boxed,vlined, noend]{algorithm2e}

\usepackage{amsmath}
\usepackage{amsthm}

\newtheorem{lemma}{Lemma}%[section]
\newtheorem{theorem}{Theorem}%[section]
\theoremstyle{definition}

\theoremstyle{remark}

\title{Lazy Rearrangement Planning in Confined Spaces}
\author{
    Rui Wang, Kai Gao, Jingjin Yu and Kostas Bekris
}
\affiliations{
    Rutgers University, The State University of New Jersey \\
    \{rw485, kg627, jy512, kb572\}@cs.rutgers.edu
}

\usepackage{amsfonts}
\newcommand{\arrangementspace}{\mathcal{A}}
\newcommand{\workspace}{\mathcal{W}}
\newcommand{\positionspace}{\mathcal{P}}
\newcommand{\manipulator}{\mathcal{M}}
\newcommand{\objects}{\mathcal{O}}
\newcommand{\cspace}{\mathcal{Q}}
\newcommand{\constraintspace}{\mathcal{C}}

\newcommand{\mRS}{mRS}
\newcommand{\DFSDP}{DFS_{DP}}

\newcommand{\CIRS}{CIRS}
\newcommand{\LRS}{LRS}

\begin{document}

\maketitle

\begin{abstract}
Object rearrangement is important for many applications but remains challenging, especially in confined spaces, such as shelves, where objects cannot be accessed from above and they block reachability to each other. Such constraints require many motion planning and collision checking calls, which are computationally expensive. In addition, the arrangement space grows exponentially with the number of objects. To address these issues, this work introduces a lazy evaluation framework with a local monotone solver and a global planner. Monotone instances are those that can be solved by moving each object at most once. A key insight is that reachability constraints at the grasps for objects' starts and goals can quickly reveal dependencies between objects without having to execute expensive motion planning queries. Given that, the local solver builds lazily a search tree that respects these reachability constraints without verifying that the arm paths are collision free. It only collision checks when a promising solution is found. If a monotone solution is not found, the non-monotone planner loads the lazy search tree and explores ways to move objects to intermediate locations from where monotone solutions to the goal can be found. Results show that the proposed framework can solve difficult instances in confined spaces with up to 16 objects, which state-of-the-art methods fail to solve. It also solves problems faster than alternatives, when the alternatives find a solution. It also achieves high-quality solutions, i.e., only 1.8 additional actions on average are needed for non-monotone instances. 
\end{abstract}

\section{Introduction}
\label{sec:intro}

Rearrangement in confined spaces is a critical robot skill in setups such as aligning objects in shelves (Fig. \ref{fig:real_robot_setting}), object retrieval and part assembly. These are harder setups than the less constrained tabletop case, where all objects can be directly accessed with top-down grasps and then lifted sufficiently high to avoid collisions. Tabletop rearrangement allows to ignore robot-object and object-object interactions, which leads to faster solutions and stronger guarantees. In the confined setup considered here, the robot arm has limited space to maneuver and cannot necessarily access all objects at any point, since top-down grasps are not available. Solving such tasks becomes harder even for a few objects, let alone maintaining high efficiency and good solution quality.

\begin{figure}[t!]
    \centering
    \includegraphics[width=0.46\textwidth]{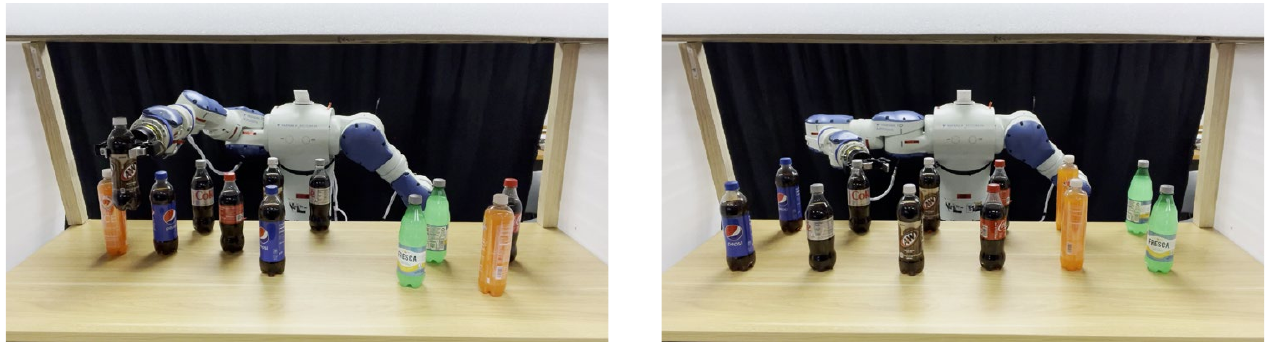}
    \caption{Real robot application of the proposed method. [Left] The robot rearranging objects. [Right] Final result.}
    \label{fig:real_robot_setting}
\end{figure}

Many existing solutions for confined rearrangement using prehensile primitives\footnote{i.e., using grasping, in contrast to non-prehensile primitives, such as pushing \cite{king2015nonprehensile, king2017unobservable, papallas2020non, vieira2022persistent}.} follow a similar high-level strategy \cite{stilman2007manipulation, wang2021uniform, wang2022efficient}. They build a global tree $T_g = (V, E)$ (Fig. \ref{fig:history_of_methods}(a)) where a node $v \in V$ (circles in Fig. \ref{fig:history_of_methods}(a)) represents an arrangement state $\alpha_v$ of the objects and an edge $e(v,u) \in E$ (arrows in Fig. \ref{fig:history_of_methods}(a)) represents a collision-free arm path so as to pick-and-place an object to transition from arrangement $\alpha_v$ to arrangement $\alpha_u$. The pick-and-place path can be computed by calling a motion planner. For non-monotone problems, some objects need to be moved to an intermediate location different from both the object's start and goal. For these instances, the global tree can be built hierarchically by concatenating subtrees $T_{sub}$ returned by a local monotone solver as partial solutions until the final arrangement $\alpha_F$ is connected to the initial arrangement $\alpha_I$, as in Fig. \ref{fig:history_of_methods}(a).

\begin{figure*}[ht]
    \centering
    \includegraphics[width=0.85\textwidth]{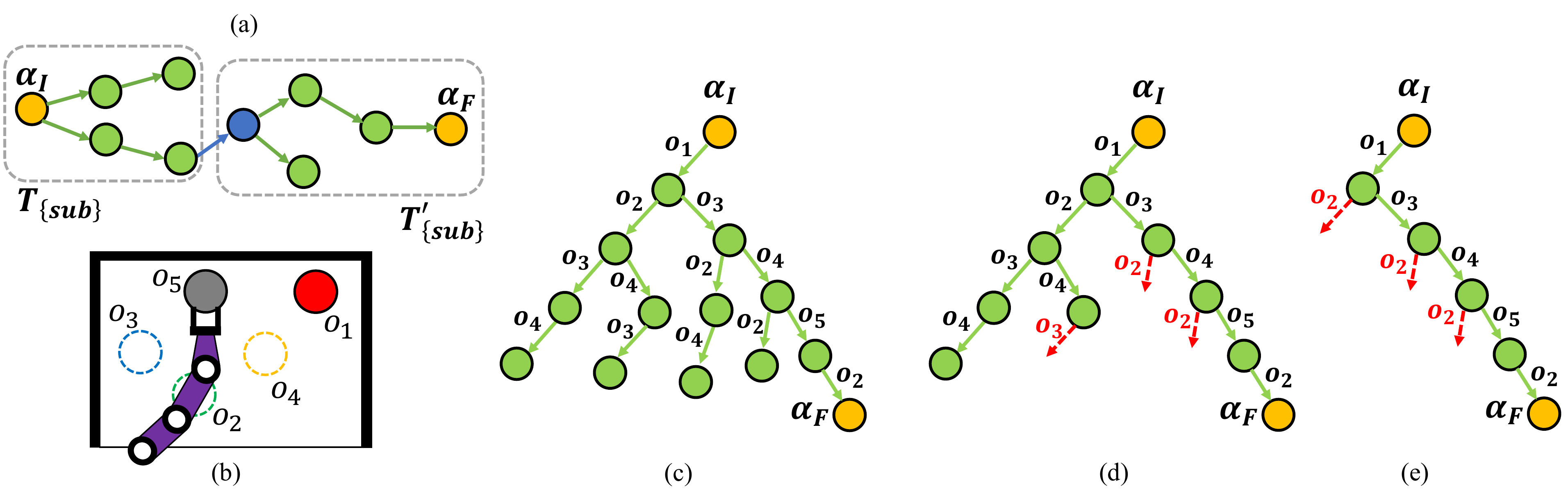}
    \caption{(a) A typical tree for non-monotone rearrangement problems, which grows incrementally local trees ($T_{sub}$, $T^\prime_{sub}$), which are monotone, i.e., an object is moved at most once and only to its goal at the final arrangement $\alpha_F$. The edge connecting the two trees (dark blue arrow) corresponds to moving an object to an intermediate location. (b) A toy example with 5 objects (start positions are denoted as solid circles while goal positions as dashed circles). The goal position of $o_2$ makes the robot arm unable to move $o_5$ after the placement of $o_2$. (c-e) The tree corresponding to $\tt \mRS$ (c), $\tt \DFSDP$ (d) and $\tt \CIRS$ (e) on the toy example. }
    \label{fig:history_of_methods}
\end{figure*}
% (d) The search tree manifested by $\tt \DFSDP$ on the toy example in (b). (e) The search tree manifested by $\tt \CIRS$ on the toy example in (b).

A basic \emph{monotone Rearrangement Solver} ($\tt \mRS$)  \cite{stilman2007manipulation} explores all possible object orderings that connect $\alpha_I$ to $\alpha_F$ via backtracking search (Fig. \ref{fig:history_of_methods}(c)). It faces scaling issues due to its $O(n!)$ time complexity. Dynamic programming solutions, such as $\tt \DFSDP$  \cite{wang2021uniform}, observe that solving a subproblem from an arrangement $\alpha$ (i.e., moving the remaining objects to their goals) does not depend on how $\alpha$ is reached from $\alpha_I$, i.e., the ordering with which the objects reach $\alpha$ does not change the feasibility of the subsequent task $\alpha \rightarrow \alpha_F$. In the toy example of Fig. \ref{fig:history_of_methods}(b), no solution is found from $\alpha$ obtained from the branch $o_1 \rightarrow o_2 \rightarrow o_3$. This can also quickly invalidate the branch $o_1 \rightarrow o_3 \rightarrow o_2$ (Fig. \ref{fig:history_of_methods}(d)). Therefore, an arrangement state is sufficient to represent all orderings of objects, which result in the same placements. Given this observation, $\tt \DFSDP$ solves monotone problems with complexity $O(2^n)$, a significant improvement over $O(n!)$.

Nevertheless, neither $\tt \mRS$, nor $\tt \DFSDP$ take advantage of constraint reasoning to further prune the search space. As robot-object constraints arise often in confined and cluttered setups, such reasoning can increase efficiency. In the toy example of Fig. \ref{fig:history_of_methods}(b), the goal position of $o_2$ makes $o_5$ unreachable, which can be identified just by considering robot configurations for grasping $o_5$. Therefore, moving $o_2$ should be treated as an invalid action for a monotone solution at any arrangement state where $o_5$ is still at its start. Such constraints can significantly prune the search space as in Fig. \ref{fig:history_of_methods}(e). Such constraints were first considered in prior related work  ($\tt \CIRS$) \cite{wang2022efficient} but not comprehensively. 

The current work focuses on utilizing such constraints to reduce the computational overhead due to the most expensive primitive, i.e., the pick-and-place motion planning calls for each edge of the search tree, which involve collision checking. All the aforementioned methods call a motion planner for every edge of the search tree during its generation to check if the transition to the child arrangement is valid. Consequently, the number of motion planning queries is equivalent to the total number of edges in the search tree, which grows exponentially as the number of objects increase.  The objective here is to avoid this combinatorial number of calls to motion planning, while still finding any solution that existing methods can discover and maintaining high-quality solutions, i.e., not requiring many intermediate locations for non-monotone problems.

If the approach detects early that there is no way to reach the final arrangement from a tree node, then significant time can be saved by not performing motion planning on edges out of that node. And it is possible to identify the infeasibility of arrangements quickly given only the objects' locations and the robot's grasps, i.e., if a set of objects in an arrangement blocks the grasps of another object at its start or its goal, then there is no arm path that can perform the pick-and-place at the corresponding arrangement (up to the resolution of grasps). These constraints are referred here as \emph{reachability constraints}. They arise often and dominate the feasibility of rearrangements in confined spaces. Moreover, computing reachability constraints is relatively easy as grasps can be obtained by  inverse kinematic (IK) solvers in a hundredth of a second; much faster than a motion planning routine for a pick-and-place sequence. In addition, given $k$ such grasps and total $n$ objects, those IK operations are performed $kn$ times ($O(n)$), which is much smaller than the worse case for the size of the search tree ($O(2^n)$).

Overall, this work proposes a lazy evaluation framework for rearrangement in confined spaces. It uses reachability constraints to quickly reveal the dependencies between objects without performing expensive motion planning queries during the search tree generation. Thus, this work: % provides the following contributions:

\begin{figure*}[ht]
    \centering
    \includegraphics[width=0.80\textwidth]{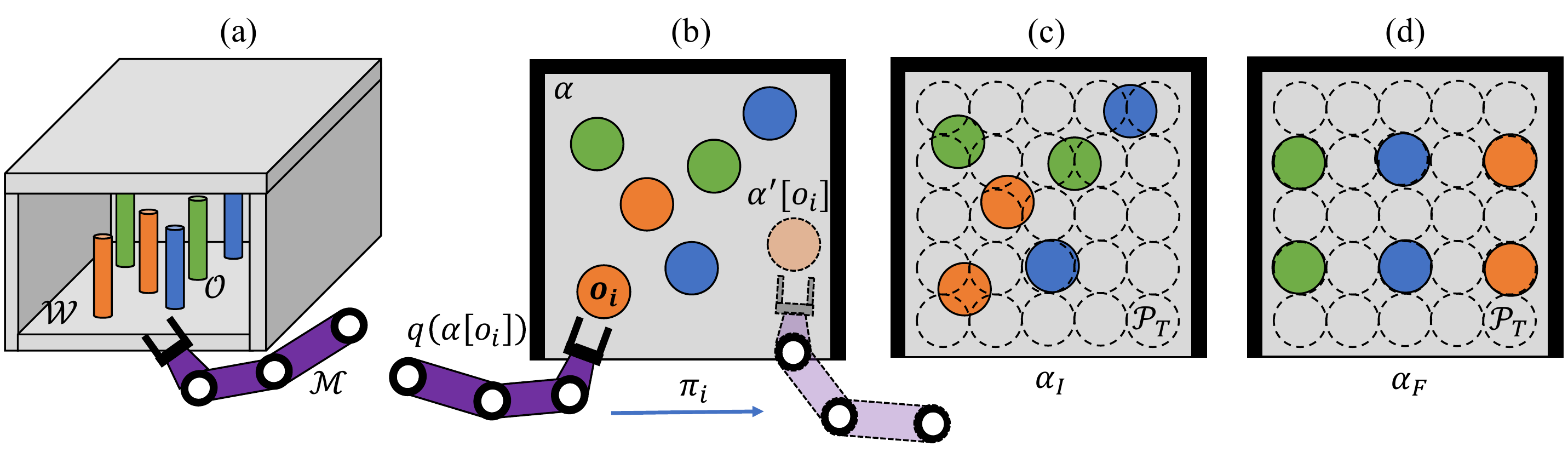}
    \caption{(a) Objects $\objects$ in a workspace $\workspace$, where the arm $\manipulator$ approaches $\objects$ from one side of $\workspace$. (b) The arm moves object $o_i$ at $p_i = \alpha[o_i]$ to a new position $p_i^{\prime} = \alpha^{\prime}[o_i]$ following the arm path $\pi_i$. (c) The initial arrangement $\alpha_I$. (d) A specified final arrangement $\alpha_F$ (e.g., objects' goals are chosen from a set of positions $\positionspace_{T}$ so that same color objects are in the same column).} 
    \label{fig:problem_formulation}
\end{figure*}

\noindent 1. \textbf{Proposes an efficient, lazy local monotone solver that utilizes reachability constraints.} Given the constraints that arise from grasp reachability, the method first builds a lazy search tree that respects these constraints without verifying edge validity via motion planning. It only performs such verification when a promising solution is found.

\noindent 2. \textbf{Proposes a global planner for non-monotone problems with high success rate, efficiency and quality.} The global planner loads the local search tree from the monotone solver and explores how to move objects to intermediate spots, from where a local solver can be called to discover monotone solutions to the goal. Edges are checked lazily only if needed for moving an object to an intermediate position. 

The overall framework is demonstrated to scale up to 16-object instances in confined spaces, which state-of-the-art alternatives fail to solve. The method significantly improves computational efficiency while maintaining high solution quality, i.e., only 1.8 additional actions are needed to fulfill non-monotone tasks on average. The solution of the proposed method is executed on a real robotic system.

\section{Related Work}
\label{sec:related}

Object retrieval in confined setups requires relocating some objects to retrieve an object. The challenges lie in (1) partial observability \cite{xiao2019online}, (2) what and where to relocate obstacles \cite{ahn2021integrated} and (3) which robot relocates which object \cite{ahn2021coordination}. The rearrangement problem considered here treats each object as a target to be relocated from its initial to final position and tends to be harder than object retrieval of a single target. 

Rearrangement planning without violating constraints can be modeled as a constraint satisfaction problem (CSP), where a variable corresponds to a moving action at a time step and the value domain refers to a set of objects to move at that time step. The solution corresponds to an assignment of the variables that respects the constraints. \cite{haralick1980increasing} discusses ways to increase tree efficiency to solve binary CSP while \cite{dechter2003constraint} focuses more on constraints processing. \cite{bartak2010constraint} introduces definitions and techniques of constraint satisfaction in general AI planning and scheduling, many of which are adopted in this work for consistency. Constraint satisfaction techniques are increasingly used in robot planning problems. \cite{havur2014geometric} uses Answer Set Programming (ASP) to decompose the cluttered workspace to place objects. Forward state-space search algorithms are developed to guide robotic spatial extrusion tasks \cite{garrett2020scalable} under geometric and stiffness constraints. \cite{krontiris2016efficiently} introduces a dependency graph describing object constraints and compute solutions via topological sorting. Such an approach, however, can be computational expensive as it can produce many dependency graphs, each of which is computed for a unique combination of grasping configurations and arm paths for manipulating objects.
% into a minimum number of non-uniform grid cells where the objects can be located. 

The proposed lazy evaluation framework delays path verification until a seemingly feasible plan is found. There is extensive work on lazy evaluation for robot path planning, where edge evaluations are expensive, such as lazy variants of the weighted A* \cite{Cohen2015PlanningSM} that postpone expensive operations in N-arm robot problems, a Lazy Receding Horizon A* (LRA*) that balances edge evaluations and graph operations \cite{mandalika2018lazy} and a lazily evaluated Lifelong Planning A* (LPA*) that reduces excessive edge evaluations of LPA* \cite{lim2021lazy}. A generalized lazy search (GLS) framework is introduced to toggle between search and edge evaluation \cite{mandalika2019generalized}. Lazy collision checking is also used in sampling-based motion planners such as Probabilistic Roadmap (PRM) \cite{kavraki2000path, sanchez2003single}, the properties of which have been examined \cite{Hauser2015LazyCC}. These approaches aim at reducing edge evaluation on computing a path (lazy motion planning), while the proposed framework aims at reducing the cost of path verification when computing a task plan (lazy rearrangement planning).

\section{Problem Formulation}
\label{sec:problem_formulation}

There are $n$ uniformly-sized cylindrical objects $\objects=\{o_1, \cdots, o_n\}$ residing in a cuboid, bounded workspace $\workspace \subset \mathbb{R}^3$, each of which at a position $p_i \in \mathbb{R}^2, i \in \{1,\cdots,n\}$ (Fig. \ref{fig:problem_formulation}(a)). Such an assignment of objects $\objects$ to a set of object positions $\{p_1, \cdots, p_n\}$ defines an \emph{arrangement} $\alpha \in \arrangementspace$, where $\arrangementspace$ is the arrangement space. $\alpha[o_i]=p_i$ indicates that object $o_i$ is at position $p_i$ given the arrangement $\alpha$.

A robot arm $\manipulator$ can access the objects $\objects$ from only one side of the workspace $\workspace$ and can move them one at a time. The arm $\manipulator$ acquires a \emph{configuration} $q \in \cspace$ where $\cspace$ is the space of all possible configurations that $\manipulator$ can acquire. The swept volume $V(q)$ represents the space occupied by $\manipulator$ at $q$. If $\manipulator$ is grasping an object, the swept volume also includes the object's volume given the grasp. $q(\alpha[o_i])$ represents a configuration where the arm can grasp $o_i$ at position $p_i = \alpha[o_i]$. An \emph{arm path} $\pi_i: [0,1] \to \cspace$ for an object $o_i$ corresponds to a sequence of configurations that move object $o_i$ from its current position $p_i = \alpha[o_i]$ to another position $p_i^{\prime}$, resulting in a new arrangement $\alpha^{\prime}$ where $\alpha^{\prime}[o_i] = p_i^{\prime}$ and $\forall j \in \{1,\cdots,n\}, j \neq i: \alpha^{\prime}[o_j] = \alpha[o_j]$ (Fig. \ref{fig:problem_formulation}(b)). Such a path is \emph{valid} if no collision arises between $\bigcup_{t=0}^{1}V(\pi_i(t))$ and other static geometries ($\workspace$ and static objects in $\workspace$). 

The workspace is decomposed into a set of possible positions $\positionspace_{T}$ where objects can be placed (Fig. \ref{fig:problem_formulation}(c)(d)). $\mathcal P_T$ are sampled with a resolution $\delta_r$ defined as the distance between adjacent candidates in the same row (or column). Goal positions of objects and intermediate positions called \emph{buffers} are chosen from $\positionspace_{T}$. The rearrangement problem is defined as follows: given any possible initial arrangement $\alpha_I$ (Fig. \ref{fig:problem_formulation}(c)) and a specified final arrangement $\alpha_F$ of $n$ objects $\objects$ (Fig. \ref{fig:problem_formulation}(d)), find a sequence of valid arm paths $\Pi = (\pi_0, \pi_1, \ldots, )$, which moves all objects from $\alpha_I$ to $\alpha_F$. A problem is \emph{monotone} if the sequence $\Pi$ consists of at most one arm path for each object. Otherwise, the problem is \emph{non-monotone} and at least one object needs to be moved to a buffer before being moved to its goal.

\section{Methodology}
\label{sec:methodologies}
This section introduces lazy rearrangement planning and describes: (1) how reachability constraints are generated and used to prune the search space, (2) how the local, monotone solver builds a lazy search tree, and (3) how the global planner operates over these trees for non-monotone instances.

\begin{figure}[t!]
    \centering
    \includegraphics[width=0.42\textwidth]{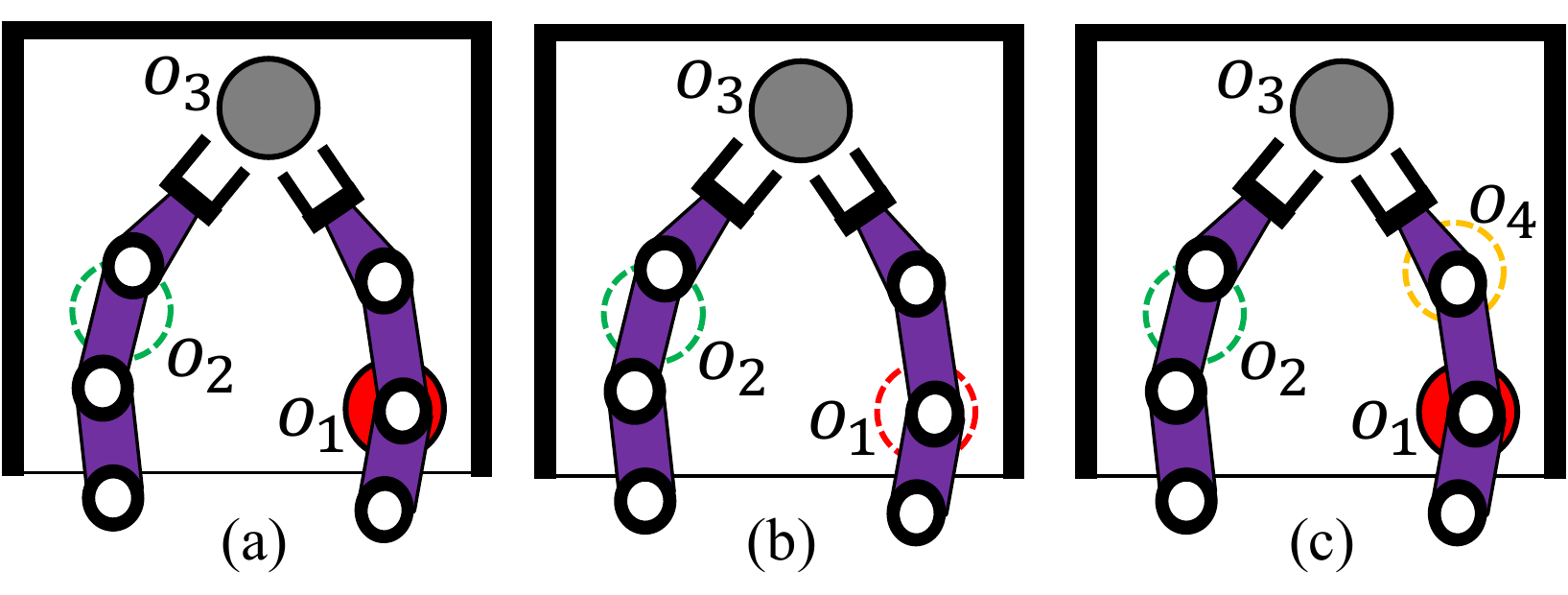}
    \caption{Three examples of the robot arm moving object $o_3$ (gray solid circle). Start and goal positions of objects are denoted as solid circles and dashed circles, respectively. For simplicity, only relevant object positions are displayed.}
    \label{fig:reachability_constraints}
\end{figure}

\subsection{Reachability Constraints}

Consider Fig. \ref{fig:reachability_constraints}(a) where there are two arm configurations for grasping $o_3$: $o_3$ is unreachable from one configuration, if $o_2$ is at its goal; while $o_3$ is unreachable from the other, if $o_1$ is at its start. Define $S_i$ as the event that $o_i$ is at its start while $G_i$ as the event that $o_i$ is at its goal. Then in Fig. \ref{fig:reachability_constraints}(a), $o_3$ is unreachable if the event $G_2 \bigcap S_1$ occurs, which generates the reachability constraint:
\begin{equation}
    \emph{Don't move $o_3$, if $o_2$ is at its goal and $o_1$ is at its start.}  
\end{equation}
In Fig. \ref{fig:reachability_constraints}(b), $o_3$ is unreachable if the event $G_2 \bigcap G_1$ occurs: 
\begin{equation}
    \emph{Don't move $o_3$, if $o_2$ is at its goal and $o_1$ is at its goal.}
\end{equation}
Note that when all reachability constraints arise from goal positions, additional constraints can be elicited for the constraining objects, i.e., $o_2$ and $o_1$ in this example. In particular, the additional constraints are: 
\begin{equation}
    \emph{Don't move $o_2$, if $o_1$ is at its goal and $o_3$ is at its start.}
\end{equation}
\begin{equation}
    \emph{Don't move $o_1$, if $o_2$ is at its goal and $o_3$ is at its start.}
\end{equation}
For generalization, one more example (Fig. \ref{fig:reachability_constraints}(c)) is given where a grasping configuration intersects with more than one object position (start position of $o_1$ and goal position of $o_4$). In this case, $o_3$ is unreachable if the event $G_2 \bigcap (S_1 \bigcup G_4)$ occurs. The event can be decoupled as
\begin{equation}
    G_2 \bigcap (S_1 \bigcup G_4) = (G_2 \bigcap S_1) \bigcup (G_2 \bigcap G_4)
\end{equation}
where $G_2 \bigcap S_1$ can be handled similar to the constraint (1) of Fig. \ref{fig:reachability_constraints}(a),  while $G_2 \bigcap G_4$ can be handled similar to constraint (2) of Fig. \ref{fig:reachability_constraints}(b).

The above reachability constraints readily generalize to $k$ grasping configurations and $n$ objects. All the reachability constraints can be stored in a container $\constraintspace$ to indicate invalid actions of moving certain objects at a certain arrangement.

Grasping configurations can be acquired via an IK solver. An improved version of the Samuel Buss IK library was used \cite{buss2005selectively}. In practice, different IK solutions place the last few links close to the gripper in similar locations, inducing similar reachability constraints. Given that, it usually suffices to consider only one IK per grasp.

\subsection{Monotone Instances}

A lazy rearrangement solver ($\tt \LRS$) is proposed to solve local, monotone tasks $\alpha_S \rightarrow \alpha_F$ where $\alpha_S$ is the start arrangement.
Alg. 1 summarizes the two steps of the solver: (1) obtain reachability constraints $\constraintspace$ from the task (Line 1); (2) use these constraints to grow lazily a search tree rooted at $\alpha_S$ (Line 2). Two definitions are introduced here to avoid potential confusion: In the context of a lazy tree, a node $\alpha_j$ is \emph{connected} to another node $\alpha_i$ if the edge $\alpha_i \rightarrow \alpha_j$ does not violate reachability constraints. In contrast, $\alpha_j$ is \emph{accessible} from $\alpha_i$ if $\alpha_i \rightarrow \alpha_j$ is verified by a motion planner to exist. Connectivity is a prerequisite for accessibility. 

\begin{algorithm}[t!]
\label{alg:LRS}
\DontPrintSemicolon
\begin{small}
\KwIn{$\alpha_S$, $\alpha_F$, $\objects$}
\KwOut{$T_{sub}$}
\SetKwComment{Comment}{\% }{}
\caption{Lazy Rearrangement Solver ($\tt \LRS$)}
\SetAlgoLined
    $\constraintspace = \textsc{ObtainConstraints}(\alpha_S, \alpha_F, \objects)$\\
    \textbf{return} $T_{sub}, flag = \textsc{GrowLocalTree}(\emptyset, \alpha_S, \alpha_F, \constraintspace)$\\
\end{small}
\end{algorithm}

Step 2 of Alg. 1 is detailed in Alg. 2, which is a recursive routine to grow a local tree $T_{sub}$ from $\alpha_C$ to $\alpha_F$ that respects the reachability constraints $\constraintspace$. Here $\alpha_C$ refers to the arrangement where the search process is currently at. The output refers to a local, monotone tree $T_{sub}$ and a flag indicating if a solution is found (initially set to false, Line 1). 

% The set $\overline{\objects}(\alpha_C)$ stores the objects not yet at their goals according to $\alpha_C$.
For each object $o \in \overline{\objects}(\alpha_C)$ not yet at its goal given $\alpha_C$ (Line 2), a forward-checking routine is performed (Line 3) to evaluate if moving $o$ violates constraints $\constraintspace$ at $\alpha_C$. If it does, then this action is pruned at $\alpha_C$ (Line 4, also see Fig. \ref{fig:lazy_evaluation_key_steps}(a)). Therefore, $\tt LRS$ identifies dead-ends and prunes the search space by reducing the action space. If $\constraintspace$ are satisfied, the resulting node $\alpha_{new}$ is generated (Line 5) where $\alpha_{new}[o] = \alpha_F[o], \alpha_{new}[\objects \setminus \{o\}] = \alpha_C[\objects \setminus \{o\}]$. If $\alpha_{new}$ is not in the tree, it will be added to the tree lazily (Line 6-7) without path verification to save computation. The tree keeps growing lazily in a recursive manner (Line 13) until $\alpha_F$ is connected to $\alpha_S$ via a branch (Line 8, also see Fig. \ref{fig:lazy_evaluation_key_steps}(b)) where the path verification will be performed on the edges of the branch (Line 9), which is described in Alg. 3. 

A global parameter $mode$ is used to indicate the current search mode (backtracking or backjumping). It is in backtracking mode when Alg. 2 is first called in step 2 of Alg. 1. If the branch is valid, the solution is found and it quits Alg. 2 with backtracking (Line 11, 13-14). If not, the search switches to the backjumping mode (Line 10), in which the search tree jumps back to the last accessible node $\alpha_{last}$ to continue (Line 15-17, also see Fig. \ref{fig:lazy_evaluation_key_steps}(e)). The tree is returned when all alternatives have been exhausted (Line 18). 

\begin{algorithm}[t!]
\label{alg:grow_local_tree}
\DontPrintSemicolon
\begin{small}
\KwIn{$T_{sub}$, $\alpha_C$, $\alpha_F$, $\constraintspace$}
\KwOut{$T_{sub}$, $flag$}
\SetKwComment{Comment}{\% }{}
\caption{$\textsc{GrowLocalTree}$}
\SetAlgoLined
    $flag = False$ \\
    \For{$o \in \overline{\objects}(\alpha_C)$}{
        $valid = \textsc{ForwardChecking}(o, \alpha_C, \constraintspace)$ \\
        \lIf{not $valid$}{continue}
        $\alpha_{new} = \textsc{getNewNode}(o, \alpha_C, \alpha_F)$ \\
        \lIf{$\alpha_{new} \in T$}{continue}
        $T_{sub} = \textsc{AddNode}(\alpha_{new}, \alpha_C, T_{sub})$ \\
        \If{$\alpha_{new} = \alpha_F$}{
            $T_{sub}, success, \alpha_{last} = \textsc{VerifyBranch}(\alpha_{new},T_{sub})$\\
            \lIf{not $success$}{$mode = $ backjumping}
            \Return $T_{sub}, success$
        }
        \Else{
            $T_{sub}, flag = \textsc{GrowLocalTree}(T_{sub}, \alpha_{new}, \alpha_F, \constraintspace)$ \\
            \lIf{$flag$}{\Return $T_{sub}, flag$}
            \If{$mode = $ backjumping}{
                \lIf{$\alpha_{new} = \alpha_{last}$}{$mode = $ backtracking}
                \lElse{\Return $T_{sub}, flag$}
            }
        }
    }
    \Return $T_{sub}, flag$
\end{small}
\end{algorithm}

\begin{figure}[b!]
    \centering
    \includegraphics[width=0.42\textwidth]{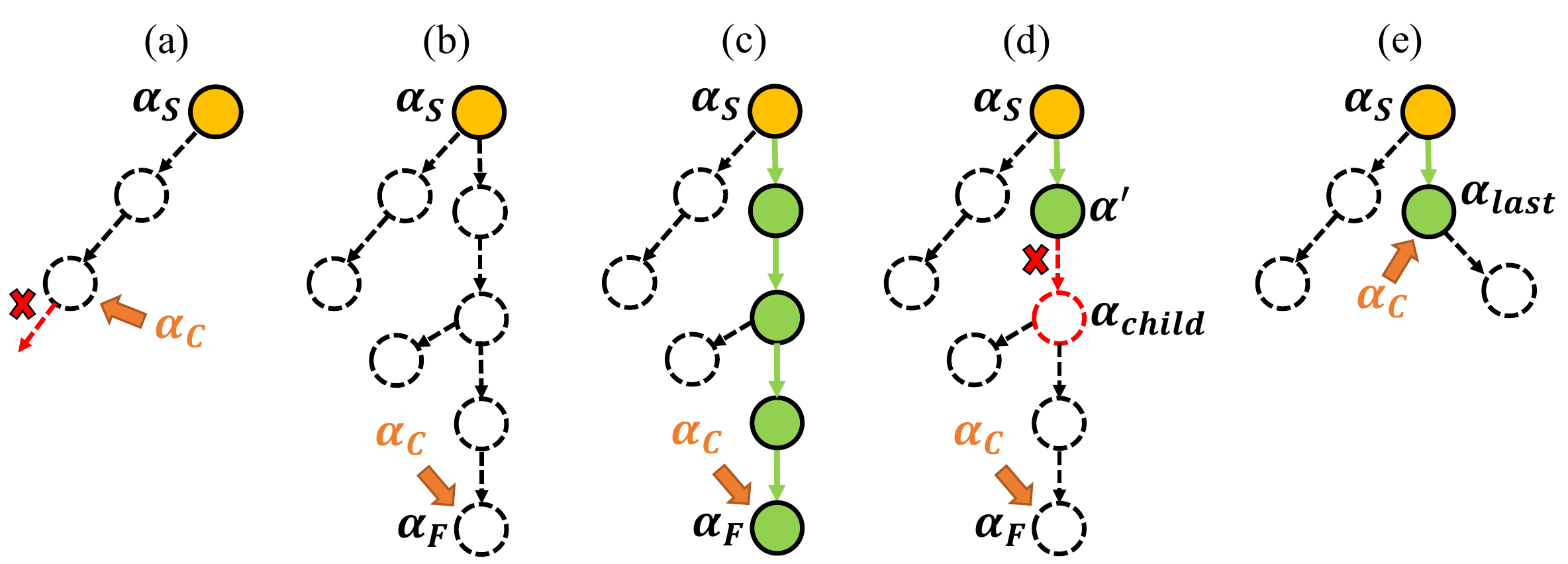}
    \caption{Lazy tree in $\tt LRS$. 
    % Nodes and edges without path verification are denoted as black dashed circle and arrows while those with path verification are denoted as green solid circles and arrows, respectively. 
    Black dashed circles and arrows represents nodes and edges without path verification while green solid ones represents those with path verification.
    Here $\alpha_C$ indicates the current stage of the search. (a) Forward-checking phase (Alg. 2, Line 3-4). A tree cannot expand upon $\alpha_C$ due to constraint violation. (b) The tree connects $\alpha_S$ to $\alpha_F$ lazily (Alg. 2, Line 8). (c) The branch pass path verification (Alg. 3, Line 11). (d) Path verification fails at the edge $\alpha^{\prime} \rightarrow \alpha_{child}$ and is terminated. (e) The subtree rooted at $\alpha_{child}$ is trimmed (Alg. 3, Line 9-10) and the search jumps back to the last accessible node $\alpha_{last}$ and backtracks to grow the tree lazily (Alg. 2, Line 15-17). } 
    \label{fig:lazy_evaluation_key_steps}
\end{figure}

\begin{algorithm}[t!]
\label{alg:verify_branch}
\DontPrintSemicolon
\begin{small}
\KwIn{$\alpha_C$, $T_{sub}$}
\KwOut{$T_{sub}$, $success$, $\alpha_{last}$}
\SetKwComment{Comment}{\% }{}
\caption{$\textsc{VerifyBranch}$}
\SetAlgoLined
    $\alpha^{\prime} = \textsc{NearestNode}(\alpha_C, T_{sub})$ \\
    \While{$\alpha^{\prime} \neq \alpha_C$}{
        $\alpha_{child} = \alpha^{\prime}.child$ \\
        $success = \textsc{PathVerification}(\alpha^{\prime}, \alpha_{child})$ \\
        \If{$success$}{
            $T_{sub} = \textsc{AddEdge}(\alpha^{\prime}, \alpha_{child}, T_{sub})$ \\
            $\alpha^{\prime}=\alpha_{child}$ \\
        }
        \Else{
            $\textsc{DeleteTree}(\alpha_{child}, T_{sub})$ \\
            \Return $T_{sub}, success, \alpha^{\prime}$
        }
    }
    \Return $T_{sub}, success, \alpha^{\prime}$
\end{small}
\end{algorithm}

Alg. 3. verifies the branch that connects $\alpha_F$ from $\alpha_S$. It first traces back to the nearest node $\alpha^{\prime}$ that is accessible from the root (Line 1) to avoid duplicate edge verification. From there, it checks the validity of each edge (Line 3-4). If an edge $\alpha^{\prime} \rightarrow \alpha_{child}$ is valid (Line 5), it marks $\alpha_{child}$ as accessible from $\alpha^{\prime}$ (Line 6) and moves on to the next edge (Line 7). If not (Line 8), the branch verification stops at the edge $\alpha^{\prime} \rightarrow \alpha_{child}$ (see Fig. \ref{fig:lazy_evaluation_key_steps}(d)) and deletes any subtree rooted at $\alpha_{child}$ (Line 9) as $\alpha_{child}$ is not accessible from $\alpha^{\prime}$. It returns the trimmed lazy tree $T_{sub}$ and records the last accessible node $\alpha_{last} = \alpha^{\prime}$ (Line 10, also see Fig. \ref{fig:lazy_evaluation_key_steps}(e)). If all the edges are valid, the task is solved and the branch on $T_{sub}$ is the solution (Line 11, also see Fig. \ref{fig:lazy_evaluation_key_steps}(c)).

\begin{theorem}[Completeness of $\tt \LRS$]\label{thm:local_solver_complete}
% $\tt \LRS$ is complete for any monotone problem given a complete motion planner. ****commented by Rui*****
$\tt \LRS$ achieves the same completeness properties as the underlying motion planner. 
% and up to ter object.
\end{theorem}

\begin{proof}
Assume there is a monotone solution $\Pi^{\prime}$ but $\tt \LRS$ fails to find it. 
% In the state space of $\tt \LRS$, objects are either at starts or goals. 
Since $\Pi^{\prime}$ is monotone, each object moves directly from its start to its goal. Following the order, there is a corresponding path $P'$ in the state space of $\tt \LRS$, which is a sequence of states with the first state being $\alpha_S$ and the others being the resulting states of actions in $\Pi^{\prime}$ in the order.
% According to the monotonicity of $\Pi^{\prime}$, there is a path $P'$ corresponding to $\Pi^{\prime}$ in the state space of $\tt \LRS$.
% To prove completeness, it is suffice to prove that all the edges in $P'$ are in $T_{sub}$. Assuming the contract, 
Given the failure assumption, let the first edge in $P'$ that is not in $T_{sub}$ be $\alpha_1\rightarrow \alpha_2$.
% Since $\alpha_1\rightarrow \alpha_2$ is valid and executable by the robot, this edge respects constraints $\mathcal C$ and will be discovered by a complete motion planner. ****commented by Rui*****
Since $\alpha_1\rightarrow \alpha_2$ is a valid edge that respects constraints $\mathcal C$, it should be eventually discovered by a motion planner (e.g., assuming probabilistic completeness). Therefore, the only reason that $\alpha_2$ is not linked to $\alpha_1$ when $\tt \LRS$ explores $\alpha_1$ is that $\alpha_2$ corresponds to a node that has been considered already (with a different parent than $\alpha_1$) and deemed as a “dead end”, i.e., can't be connected to $\alpha_F$.

If $\alpha_2=\alpha_F$, then it means $\tt \LRS$ finds a path from $\alpha_I$ to $\alpha_F$ and the path succeeds in the branch verification process, which contradicts the failure assumption. Otherwise, since 
% $\alpha_1$ is a non-descendant node of $\alpha_2$ (according to the structure of the state space) and 
the search tree is developed in a depth first manner, 
it suggests that $\tt \LRS$ cannot find a path in the state space from $\alpha_2$ to $\alpha_F$. But there is a sub-path $P''$ in $P'$ from $\alpha_2$ to $\alpha_F$. 
Recursively, all the edges in $P'$ are in $T_{sub}$ or there is another path from $\alpha_I$ to $\alpha_F$ in $T_{sub}$.

Thus, given a monotone rearrangement instance, whether $\tt \LRS$ is able to find a solution depends on whether the underlying motion planner can find one. In other words, the proposed $\tt \LRS$ does not degrade the properties of the underlying motion planner (e.g., probabilistic/resolution completeness). For instance, if a solution has not been found but exists, a straightforward iterative strategy can provide to the motion planner additional resources (more sampled configurations, object position candidates, grasps and IK solutions) and then the task planner is called again.\end{proof}

The motion planner used is the asympt. optimal Probabilistic Roadmap (PRM*) \cite{karaman2011sampling}.

% Therefore, to prove completeness of $\tt \LRS$, it is suffice to prove all edges in $P''$ will exist in $T_{sub}$.
% Assuming the contract 

% According to Alg. 2 and 3, $\tt \LRS$ only gives up edges which violates reachability constraints (Alg.2, Line 2-3) or subtrees which is rooted at an inaccessible node (Alg. 3, Line 9). Given that, $\Pi^{\prime}$ must contain at least an edge $e^{\prime}$ that either (i) violates reachability constraints (ii) is a part of a subtree rooted at an inaccessible node. For (i), the edge is not valid as respecting reachability constraints is a prerequisite of admitting a valid path. For (ii), $\alpha_F$ cannot be accessible from $\alpha_S$ via $e^{\prime}$. As a result, $\Pi^{\prime}$ containing $e^{\prime}$ is not a solution. Conflict received. Therefore, $\tt \LRS$ finds a solution if it exists.

% On the other side, suppose a solution does not exist but $\tt \LRS$ finds one $\Pi^{\prime}$. Then all arm paths on $\Pi^{\prime}$ pass branch verification in Alg. 3. Therefore, $\alpha_F$ is accessible from $\alpha_S$ via $\Pi^{\prime}$ and $\Pi^{\prime}$ is indeed a solution. Conflict received. Therefore, $\tt \LRS$ will not return a solution if it does not exist.

\subsection{Non-Monotone Instances}

For non-monotone problems where some objects need to be moved to buffers first to admit a feasible solution, a global planner is needed, as the local solver only examines the monotonicity of a local task. The global planner assigns a local task to the local solver and concatenates the local tree $T_{sub}$ returned by the local solver to build a global tree $T_g$. If the resulting $T_g$ does not lead to a solution, an action is needed to move an object to a buffer at an arrangement. Such actions are called \emph{perturbations} in the context of the global planner. A new tree node $\alpha_{pert}$ (the blue node in Fig. \ref{fig:history_of_methods}(a)) will be generated upon perturbation, from which a new local task $\alpha_{pert} \rightarrow \alpha_F$ is assigned to the local planner.

\begin{figure}[b!]
    \centering
    \includegraphics[width=0.26\textwidth]{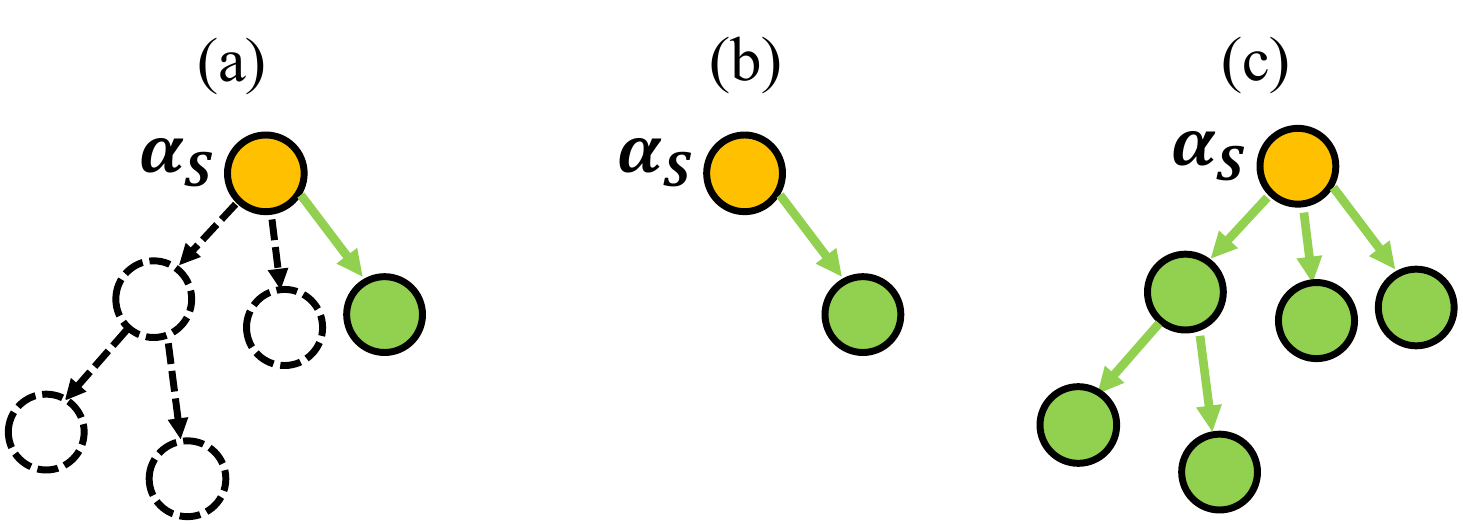}
    \caption{(a) A lazy search tree returned by the local solver $\tt LRS$ with both verified (green solid arrows) and unverified (black dashed arrows) edges. (b) The tree concatenated by a greedy global planner. (c) The tree concatenated by a conservative global planner. } 
    \label{fig:lazy_tree_concatenation}
\end{figure}

Consider a lazy local tree $T_{sub}$ in Fig. \ref{fig:lazy_tree_concatenation}(a). There are different ways of concatenating $T_{sub}$ to the global tree $T_g$.  A greedy global planner concatenates only the part of the tree which has been path verified (Fig. \ref{fig:lazy_tree_concatenation}(b)), while a conservative global planner first verifies all unverified edges in the lazy tree and then concatenates the verified tree (Fig. \ref{fig:lazy_tree_concatenation}(c)).  In the greedy case, the planner operates lazily by only accepting the verified part of the tree to avoid performing any expensive path verification. As a result, it quickly concatenates the tree but giving up unverified nodes which could lead to a solution given a proper perturbation. In the extreme cases where no nodes except the root of the local tree is accepted, the only way the tree can grow in the greedy case is through perturbation at the root. But a perturbation is both slow (just extends one edge) and unpredictable (the object is moved to a random buffer). 
In the conservative case, the planner performs additional work to verify the unverified part of the tree, thus maintaining those nodes in the global tree. This, however, undermines the utility of lazy evaluation as the global planner takes time to make up the laziness of the local planner. The only advantage lies in the last iteration if a solution is found (in that case, the global planner only needs to concatenate the verified branch).

\begin{algorithm}[h]
\label{alg:LRS_hybrid}
\DontPrintSemicolon
\begin{small}
\KwIn{$\alpha_I$, $\alpha_F$, $\objects$, $\positionspace_{T}$}
\KwOut{$\Pi$}
\SetKwComment{Comment}{\% }{}
\caption{$\tt LRS_{hybrid}$}
\SetAlgoLined
    $T_g = \emptyset$, $\Pi = \emptyset$\\
    $T_{sub} = {\tt \LRS}(\alpha_I, \alpha_F, \objects)$\\
    $T_g = T_g + T_{sub}$\\
    \While{$\alpha_F \notin T_g$ and \textsc{TimePermitted}}{
        $\alpha_C = \textsc{SelectNode}(T_g)$\\
        $T_g, success, \alpha_{last} = \textsc{VerifyBranch}(\alpha_C,T_g)$\\
        \If{$success$}{
            $\alpha_{pert} = \textsc{PerturbNode}(\alpha_C, \alpha_F, \positionspace_{T})$\\
            \If{$\alpha_{pert} \neq \emptyset$}{
                $T_g = \textsc{AddNode}(\alpha_{pert}, \alpha_C, T_g)$\\
                $T_g = T_g + {\tt \LRS}(\alpha_{pert}, \alpha_F, \objects)$\\
            }
        }
    }
    \If{$\alpha_F \in T$}{$\Pi = \textsc{TraceBackPath}(T_g, \alpha_F, \alpha_I)$}
    \Return $\Pi$\\
\end{small}
\end{algorithm}

To leverage the benefits of both the greedy and the conservative planner, this work introduces a global planner $\tt LRS_{hybrid}$, which concatenates the lazy local tree as it is and only performs path verification if the node being selected for perturbation is unverified. In this manner, the global tree maintains both verified and unverified edges (hybrid) throughout the search. $\tt LRS_{hybrid}$ is described in Alg. 4. It solves a global task $\alpha_I \rightarrow \alpha_F$ with $n$ objects and all possible positions $\positionspace_{T}$ and outputs a path sequence $\Pi$. The global task is first assigned to the local solver, from where the solver tests the monotonicity of the problem (Line 1-3). If the problem is non-monotone (Line 4), a node $\alpha_C$ is selected to perform a perturbation (Line 5, also see Fig. \ref{fig:perturbation_process}(a)). Before perturbation, $\alpha_C$ is checked to see if it is accessible from the root $\alpha_I$ (Line 6). If not, the perturbation terminates and restarts (Fig. \ref{fig:perturbation_process}(c)(d)). If $\alpha_C$ is accessible from $\alpha_I$ (Line 7), the perturbation is performed on $\alpha_C$ by randomly selecting an object to be placed in a buffer, which is randomly selected from $\positionspace_{T}$ (Line 8). If the perturbation is not successful, it terminates and restarts (Line 4). Otherwise (Line 9), a perturbation node $\alpha_{pert}$ is added to the tree (Line 10, also see Fig. \ref{fig:perturbation_process}(b)), from where the local solver is called to solve a local task $\alpha_{pert} \rightarrow \alpha_F$. The local solver call and the perturbation alternate until a solution is found (Line 12-13) or the time exceeds a threshold (Line 4).

\begin{figure}[b!]
    \centering
    \includegraphics[width=0.38\textwidth]{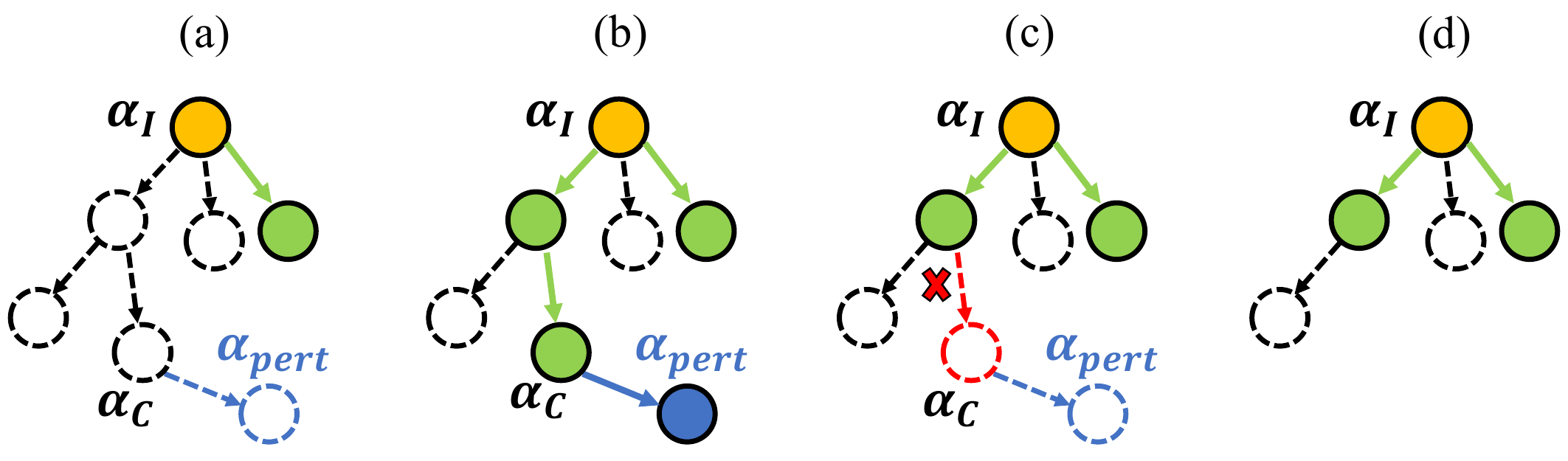}
    \caption{(a) A node $\alpha_C$ is selected to perform perturbation on a tree in Fig. \ref{fig:lazy_tree_concatenation}(a). (b) A perturbation succeeds only if (1)$\alpha_C$ is accessible from the root (green arrows) and (2) the perturbation node $\alpha_{pert}$ is accessible from $\alpha_C$ (blue arrow). (c) An unsuccessful perturbation given that $\alpha_C$ is not accessible. (d) The tree after trimming the subtree rooted at $\alpha_C$. } 
    \label{fig:perturbation_process}
\end{figure}

Despite different ways (greedy, conservative or hybrid) of tree concatenation, the planner generally follows the structure of alternating between random perturbations and sub-tree concatenation. The high-level structure is referred to as Perturbation Search (PERTS). The following discussion provides a proof that PERTS guarantees probabilistic completeness, regardless of how local trees are concatenated or the type of monotone solver used. The only requirement is that the clearance of paths $\delta$ in the rearrangement plan is lower bounded by a constant. Here $\delta$ is measured as
the minimum distance between the robot arm and the obstacles in the plan. A lemma for the completeness theorem is proved as follows.

\begin{lemma}\label{lm:complete}
Given the set of position candidates $\mathcal P_T$ with resolution $\delta_r$ and a complete motion planner,
if there is a feasible rearrangement plan $\Pi$ from $\alpha_I$ to $\alpha_F$ with clearance $\delta_{\Pi} > \sqrt{2} \delta_r$,
then there is a feasible path $P$ from $\alpha_I$ to $\alpha_F$ in the state space of the proposed global planner.
\end{lemma}

\begin{proof}
The lemma can be proven by moving all the placing positions in $\Pi$ (goal and buffer positions) to their nearest neighbors in $\mathcal{P}_T$.
Given the resolution $\delta_r$, each point in the rectangular region is at most $\dfrac{\sqrt{2}}{2} \delta_r$ away from its nearest neighbor in $\mathcal{P}_T$.
With a complete motion planner, there is a feasible path $P$ in the state space $\mathcal{S}$ of the proposed global planner given clearance: 
$\delta_{\Pi} \geq \delta-2*\dfrac{\sqrt{2}}{2} \delta_r>0$.
\end{proof}

\begin{theorem}[Probabilistic Completeness of PERTS]\label{thm:complete}
Given a complete motion planner and a constant $\delta^*$, 
if there is a rearrangement plan with clearance $\delta\geq \delta^*$, 
then the probability of the global planner to find a feasible rearrangement plan approaches 1 as the number of perturbations increase, 
regardless of which monotone solver is used.
\end{theorem}

% \begin{proposition}
% Given a complete motion planner and a fixed tolerating clearance $\delta*$, if there is a feasible rearrangement plan $\Pi$ with clearance $\delta(\Pi)<\delta*$, then PERTS  regardless of which monotone planner is used, PERTS can .
% \end{proposition}

\begin{proof}
Given the tolerating clearance $\delta^*$, generate $\mathcal{P}_T$ with resolution $\dfrac{\sqrt{2}}{3}\delta^*$. 
Assume that there is a feasible rearrangement plan $\Pi$ for a rearrangement instance.
According to the Lemma above, %~\ref{lm:complete},  
there is a feasible path $P$ from $\alpha_I$ to $\alpha_F$ in the state space $\mathcal{S}$ of PERTS.

To prove the probabilistic completeness, 
regardless of which monotone solver is used,
it suffices to prove that the probability of finding $\alpha_F$ via perturbations approaches 1 as the perturbation number $m$ goes to infinity.

In each perturbation, PERTS moves a randomly selected object from a random state to a randomly selected position in $\mathcal{P}_T$.
Since the size of the state space $\mathcal{S}$ is upper bounded by $(|\mathcal{P}_T|+n)^n$,
for all $ 2 \leq i \leq |P|$, 
when $P[i-1]$ is in the search tree, 
the probability that $P[i]$ can be added into the search tree in the next perturbation is lower bounded by
$p > 1/(|\mathcal{S}|*n*(|\mathcal{P}_T|+n))>1/(n*(|\mathcal{P}_T|+n)^{n+1})$.

Similar to the probabilistic completeness for RRT \cite{kleinbort2018probabilistic}, consider $m$ Bernoulli trials with success prob. $p$. 
Let $X_m$ be the number of successes in $m$ trials, then:

\begin{equation*}
\begin{split}
    &Pr[\alpha_F \text{ cannot be found in }m \text{ perturbations}] \leq Pr[X_m <|P|]\\
    &\leq \sum^{|P|-1}_{i=0} \begin{pmatrix} m \\ i \end{pmatrix} p^i (1-p)^{m-i} \leq \sum^{|P|-1}_{i=0} \begin{pmatrix} m \\ |P|-1 \end{pmatrix} p^i (1-p)^{m-i}\\
    &\leq \begin{pmatrix} m \\ |P|-1 \end{pmatrix} \sum^{|P|-1}_{i=0} (1-p)^{m} \leq \begin{pmatrix} m \\ |P|-1 \end{pmatrix} \sum^{|P|-1}_{i=0} (e^{-p})^{m}\\
%     \end{split}
% \end{equation*}
% \begin{equation*}
% \begin{split}
    &= \begin{pmatrix} m \\ |P|-1 \end{pmatrix} |P| (e^{-p})^{m} \leq \dfrac{|P|}{(|P|-1)!} m^{|P|-1}e^{-pm}\\
    &\rightarrow 0 (m \rightarrow \infty)
\end{split}
\end{equation*}

The transitions rely on $p<1/2$, which trivially holds when $|\mathcal P_T|>1$. Therefore, the probability of finding $\Pi$ approaches 1 as the number of perturbations increases.
\end{proof}

\section{Experiments}
\label{sec:experiments}
This section demonstrates the effectiveness of the proposed monotone solver $\tt \LRS$ on monotone problems and the global planner $\tt LRS_{hybrid}$ on non-monotone problems. The experiments are performed in a robotic simulator Pybullet, with a Motoman SDA10F robot tasked to rearrange objects in the confined setup (Fig. \ref{fig:experimental_setting}, left) and solutions of the proposed method are executed on a real robot system (Fig. \ref{fig:real_robot_setting}) \footnote{Codes and Videos are available online at \url{https://github.com/Rui1223/confined-space-rearrangement}.}.

\begin{figure}[t!]
    \centering
    \includegraphics[width=0.4\textwidth]{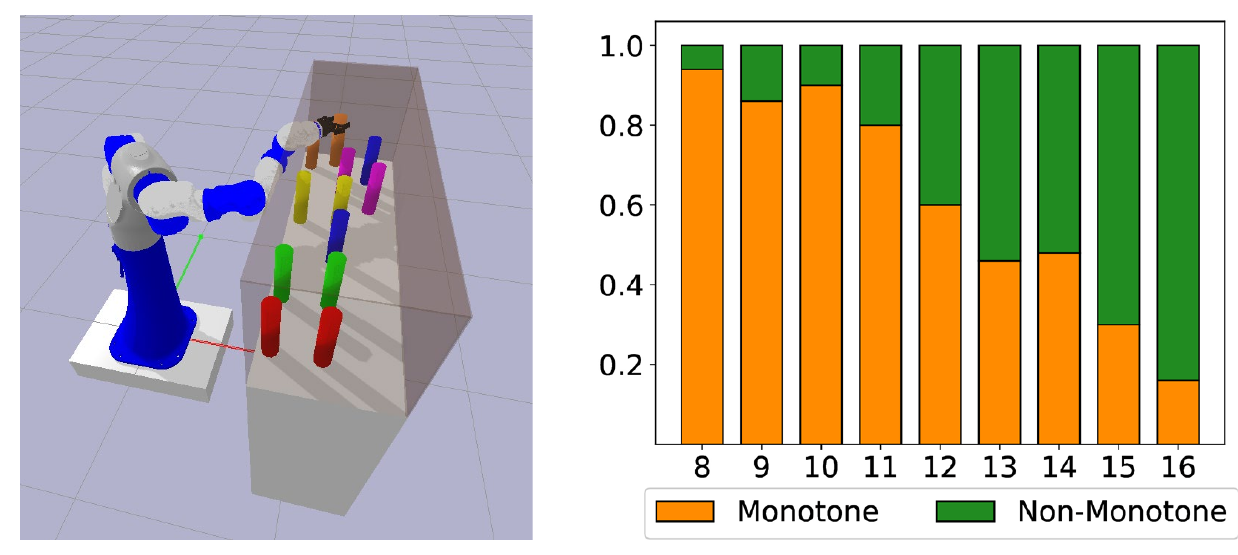}
    \caption{[Left] The robotic simulation environment where the experiments are performed. [Right] Monotone or non-monotone problem distribution per number of objects. } 
    \label{fig:experimental_setting}
\end{figure}

The problems are generated by randomly sampling start positions of the objects and assigning goal locations to these objects from a discretized set of locations in the shelf similarly to Fig. \ref{fig:problem_formulation}(c) and \ref{fig:problem_formulation}(d).
Fig. \ref{fig:experimental_setting} (right) shows the distribution of monotone and non-monotone problems for different number of objects. Given this distribution, monotone problems are selected with 8-14 objects, while non-monotone problems with 12-16 objects. 80 experiments are performed for each number of objects. The metrics for monotone and non-monotone solutions involve success rate and computation time. The computation time is separated into (1) the total computation time, and (2) the motion planning/collision checking (path verification) time. The computation time is plotted in the logarithmic scale  for better visualization.

\begin{figure}[b!]
    \centering
    \includegraphics[width=0.46\textwidth]{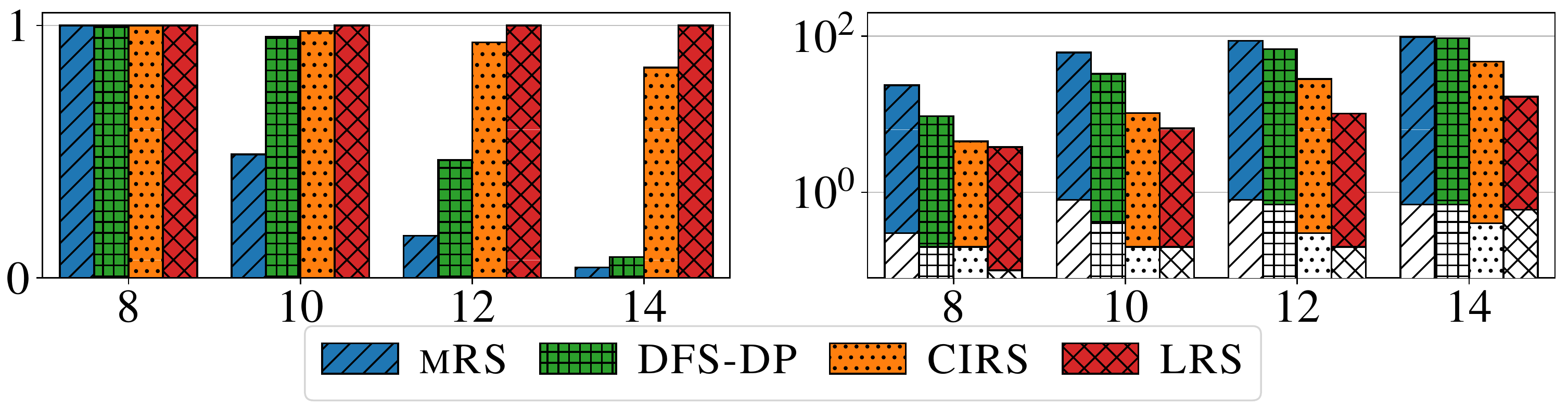}
    \caption{Algorithm comparisons in monotone instances. [left] Success rate (\%). 
    [right] Computation time (plotted in logarithm) presented as two components: the colored part is the path verification time and the white part is the other.} 
    \label{fig:monotone}
\end{figure}

{\bf Evaluation on Monotone Problems:} The proposed $\tt \LRS$ is compared to the aforementioned alternatives ($\tt \mRS$, $\tt \DFSDP$, $\tt \CIRS$). A limitation of 100 seconds is given to solve a monotone problem. Fig. \ref{fig:monotone} (left column) demonstrates that the success rate for $\tt \LRS$ stays 100\% even in harder problems (12, 14 objects) while that for $\tt \mRS$ and $\tt \DFSDP$ drops significantly (16.7\%, 4.2\% for $\tt \mRS$ and 46.7\%, 8.3\% for $\tt \DFSDP$). $\tt \CIRS$ shows relatively high success rate (83.3\% for 14 objects) but $\tt \LRS$ outperforms $\tt \CIRS$ (Fig. \ref{fig:monotone} (right column)) with faster computation time (10.2 and 16.8 seconds for 12 and 14 objects, compared to 28.3 and 47.3 seconds for $\tt \CIRS$). The lazy evaluation structure of $\tt \LRS$ significantly reduces the path verification time. For 14 objects, $\tt \LRS$ spends 16.2 seconds (96.4\% of the total time)  on path verification, while the time is 46.9 (99.2\%), 93.9 (99.3\%) and 97.3 (99.3\%) seconds for $\tt \CIRS$, $\tt \DFSDP$ and $\tt \mRS$, respectively (Fig. \ref{fig:monotone} (right column), colored part of the bars). It also confirms the observation that path verification is the computational overhead.

\begin{figure}[t!]
    \centering
    \includegraphics[width=0.46\textwidth]{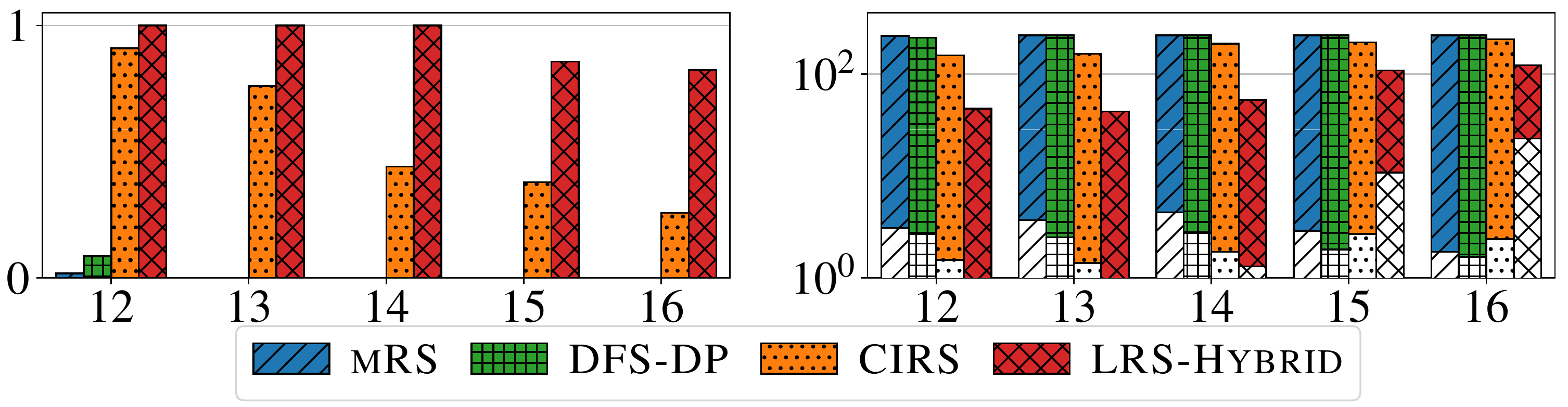}
    \caption{Algorithm comparisons in non-monotone instances. [left] Success rate (\%). [right] Computation time (logarithm) presented as two components: the colored part is the path verification time and the white part is the other. } 
    \label{fig:nonMonotoneComparison1}
\end{figure}

\begin{figure}[b!]
    \centering
    \includegraphics[width=0.46\textwidth]{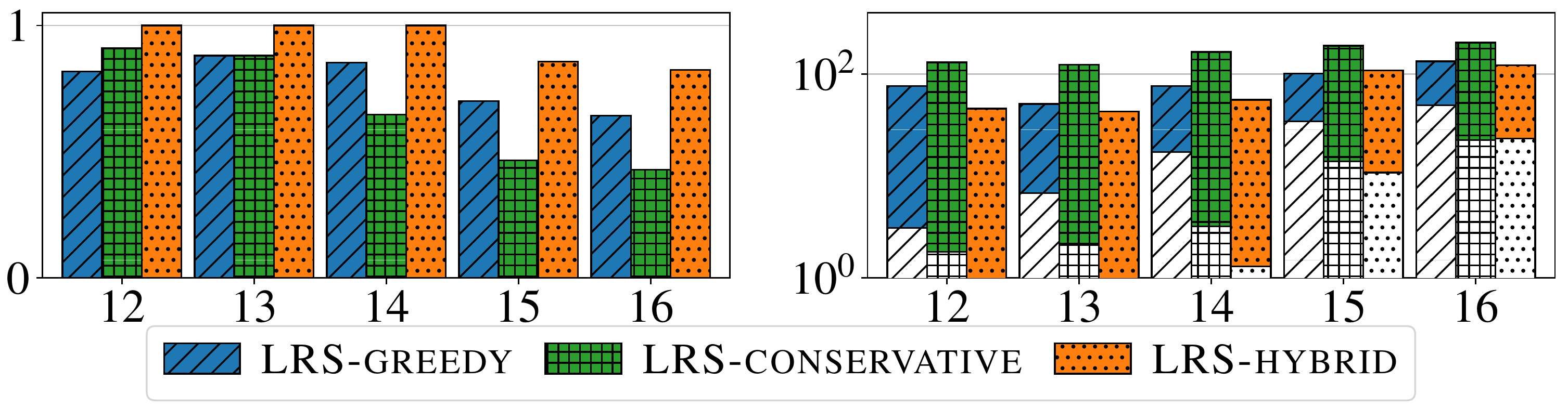}
    \caption{Ablation study in non-monotone instances. [left] Success rate (\%).
    [right] Computation time (logarithm) presented as two components: the colored part is the path verification time and the white part is the other.} 
    \label{fig:nonMonotoneComparison2}
\end{figure}

{\bf Evaluation on Non-monotone Problems:} The proposed global planner $\tt LRS_{hybrid}$ is compared to global planners which use the alternatives ($\tt \mRS$, $\tt \DFSDP$, $\tt \CIRS$) as the integrated local solver. For simplicity, the names $\tt \mRS$, $\tt \DFSDP$, $\tt \CIRS$ are still used to represent the corresponding global planner. A limitation of 240 seconds is given to solve a non-monotone problem. Fig. \ref{fig:nonMonotoneComparison1} (left) indicates that $\tt \mRS$ and $\tt \DFSDP$ fail in every instance for 13-16 objects and $\tt \CIRS$ shows a stable decrease of success rate (44.1\%, 37.9\% and 25.8\% for 14, 15 and 16 objects, respectively). In contrast, $\tt LRS_{hybrid}$ has 100\% success rate for 12-14 objects and stays high for harder problems (85.7\% for 15 objects and 82.4\% for 16 objects). $\tt LRS_{hybrid}$ solves non-monotone problems 61\% faster than alternatives, as it lazily loads the local tree and only performs path verification if the node selected for perturbation is unverified. Moreover, the time used other than path verification (Fig. \ref{fig:nonMonotoneComparison1} (right), white part of the bars) for $\tt LRS_{hybrid}$ is higher than other methods (e.g., for 16 objects, that time is 23.2 seconds for $\tt LRS_{hybrid}$, compared to 1.9 seconds for other methods), as $\tt LRS_{hybrid}$ requires steps such as obtaining constraints (Alg. 1, Line 1) and trimming inaccessible subtrees (Alg. 3, Line 9). However, the path verification time for $\tt LRS_{hybrid}$ is 98.8 seconds, much smaller than that of other methods (231.6 seconds on average). In that regard, $\tt LRS_{hybrid}$ trades heavy path verification time for some moderate increase in task planning time.

{\bf Ablation Study:} An ablation study is also performed to reveal the impact of different versions of the global planner (greedy, conservative, hybrid) on non-monotone problems. Fig. \ref{fig:nonMonotoneComparison2} (left column) shows that the success rate for $\tt LRS_{hybrid}$ is higher than $\tt LRS_{greedy}$ and $\tt LRS_{conservative}$. The two alternatives suffer from relatively low success rates for different reasons. $\tt LRS_{greedy}$ only accepts the verified part of tree to avoid path verification, giving up potential nodes. As a result, it involves more perturbation steps to grow the tree, which can be slow. $\tt LRS_{conservative}$ further verifies the unverified part of the local lazy tree before concatenating it to the global tree. As a result, it has more path verification steps, which is also slow. Fig. \ref{fig:nonMonotoneComparison2} (right column) also confirms these reasons. $\tt LRS_{greedy}$ has a smaller path verification time (colored part of the bars) but a higher other time (white part of the bars) and $\tt LRS_{conservative}$ is on the opposite. The proposed $\tt LRS_{hybrid}$ has both time in the middle compared to the other two, achieving the fastest computation time.

The number of buffers needed to solve non-monotone problems is provided in Table. 1. With the PERTS structure, the problems can be mostly solved with one or two buffers.

\begin{table}[t!]
\centering
\begin{tabular}{|c|c|c|c|c|c|}
\hline
\# objects & 12 & 13 & 14 & 15 & 16  \\ \hline
$\tt LRS_{hybrid}$ & 1.4 & 1.5 & 1.9 & 2.2 & 2.3  \\ \hline
$\tt LRS_{greedy}$ & 1.3 & 1.5 & 1.2 & 2.2 & 2.4  \\ \hline
$\tt LRS_{conservative}$ & 1.2 & 1.5 & 1.8 & 1.4 & 1.5  \\ \hline
\end{tabular}
\caption{Average number of buffers needed to solve non-monotone problems.}
\end{table}

% \begin{table}
% \centering
% \begin{tabular}{|l|l|l|l|}
% \hline
%  & $\tt LRS_{hybrid}$ & $\tt LRS_{greedy}$ & $\tt LRS_{conservative}$  \\ \hline
% 12 & 1.2 & 1.3 & 1.4  \\ \hline
% 13 & 1.5 & 1.5 & 1.5  \\ \hline
% 14 & 1.8 & 1.2 & 1.9  \\ \hline
% 15 & 1.4 & 2.2 & 2.2  \\ \hline
% 16 & 1.5 & 2.4 & 2.3  \\ \hline
% \end{tabular}
% \caption{Average number of buffers needed to solve non-monotone problems in different versions of the global planner.}
% \end{table}

\section{Conclusion and Future Work}
\label{sec:conclusion}
This work introduces a lazy evaluation framework, which involves a local monotone solver and a global planner for solving rearrangement in confined spaces. The proposed framework is capable of solving hard instances up to 16 objects with high-quality by improving computational efficiency. This paper argues the completeness of the local solver and the probabilistic completeness of the global planner. The solutions are also demonstrated on a real robotic system.

The confined setup does not only pose reachability challenges, but also visibility challenges. As the objects cannot be accessed from above by the robot arm, they cannot be always detected from above by a camera. This results in partial observability where the rearrangement objective changes, i.e., rearranging objects which improve the visibility of others. Not all objects may be known initially and replanning is needed as the scene is updated. The proposed techniques can be modified to be deployed in these setups. For instance, building a stochastic tree where the edge incorporates the probability the object is movable given uncertainty. 

\section{Acknowledgements}
\label{sec:acknow}
The work is partially supported by an NSF HDR TRIPODS award 1934924.

\bibliography{aaai22}
\end{document}